\def\year{2021}\relax
\documentclass[letterpaper]{article} 
\usepackage{aaai21}  
\usepackage{times}  
\usepackage{helvet} 
\usepackage{courier}  
\usepackage[hyphens]{url}  
\usepackage{graphicx} 
\usepackage{natbib}  
\usepackage{caption} 
\nocopyright

\usepackage{todonotes}

\usepackage{todonotes}
\presetkeys{todonotes}{inline,backgroundcolor=red!40,caption={}}{}
\usepackage{savesym}
\usepackage{amssymb,amsfonts,dsfont,amsthm}
\usepackage{amsmath}
\usepackage{outlines}
\usepackage{xcolor}
\usepackage{textcomp}
\usepackage{subcaption}
\usepackage{textcomp}
\usepackage{siunitx}
\usepackage{booktabs}
\usepackage{multirow}
\usepackage{tikz}
\usetikzlibrary{shapes.geometric}
\usepackage{algorithm}
\usepackage{algorithmic}

\DeclareMathOperator*{\Exp}{\mathbb{E}}

\newcommand{\augpolicy}{p}

\newcommand{\stochasticloss}{\ell}
\newcommand{\loss}{\mathcal{L}}
\newcommand{\stopgrad}[1]{\text{sg}(#1)}
\newcommand{\reals}{\mathbb{R}}
\newcommand{\T}{\intercal}
\newcommand{\gar}{GAR}

\theoremstyle{plain}
\newtheorem{theorem}{Theorem}

\newtheorem{lemma}[theorem]{Lemma}

\theoremstyle{definition}

\theoremstyle{remark}

\usepackage{amsmath,amsfonts,bm}

\def\eqref#1{equation~\ref{#1}}

\def\1{\bm{1}}

\DeclareMathAlphabet{\mathsfit}{\encodingdefault}{\sfdefault}{m}{sl}
\SetMathAlphabet{\mathsfit}{bold}{\encodingdefault}{\sfdefault}{bx}{n}

\newcommand{\softmax}{\mathrm{softmax}}

\DeclareMathOperator{\Tr}{Tr}

\title{In-Loop Meta-Learning with Gradient-Alignment Reward}
\author {
        Samuel M\"{u}ller,\textsuperscript{\rm 1,}\thanks{Contact Author}
        Andr{\'e} Biedenkapp,\textsuperscript{\rm 1}
        Frank Hutter,\textsuperscript{\rm 1, \rm 2}\\
}
\affiliations {
    \textsuperscript{\rm 1}Department  of  Computer  Science,  University  of  Freiburg, Germany \\
    \textsuperscript{\rm 2}Bosch Center for Artificial Intelligence, Renningen, Germany\\
    \{muellesa, biedenka, fh\}@cs.uni-freiburg.de
}
\begin{document}

\maketitle

\begin{abstract}
At the heart of the standard deep learning training loop is a greedy gradient step minimizing a given loss.
We propose to add a second step to maximize training generalization.
To do this, we optimize the loss of the next training step.
While computing the gradient for this generally is very expensive and many interesting applications consider non-differentiable parameters (e.g. due to hard samples), we present a cheap-to-compute and memory-saving reward, the \emph{gradient-alignment reward (\gar)}, that can guide the optimization.
We use this reward to optimize multiple distributions during model training.
First, we present the application of \gar{} to choosing the data distribution as a mixture of multiple dataset splits in a small scale setting.
Second, we show that it can successfully guide learning augmentation strategies competitive with state-of-the-art augmentation strategies on CIFAR-10 and CIFAR-100.
\end{abstract}

\section{Introduction}
The human capacity to learn is staggering. Not only can humans learn about the world they live in, but crucially humans can also learn a a good learning strategy simultaneously.
To this end, humans are able to learn to select their learning path. 

Take a math course as an example. In the beginning you might have followed your professor\textquotesingle{s} textbook closely.
With increasing understanding, you might have realized there are other textbooks out there, which yield better learning outcomes. In the same course you might also have learned that it can be helpful to try to prove the presented theorems yourself, instead of only reading the proof provided in the textbook.
You might have learned both of these learning strategies while learning the material of the course at the same time. So, your understanding of the material and the way you studied both improved as you studied.
Similarly, a promising route towards stronger machine learning models could lie in training not only for performing a particular task, but for better learning strategies on that task at the same time.

We define an improved learning method as a method that yields outcomes that generalize well to unseen problem instances. In a supervised learning setting that means validation error is minimized.
One way to achieve this is by optimizing meta-parameters with a meta-gradient. The meta-gradient is a gradient taken through an SGD training run, which is possible since SGD itself is differentiable. 


We define the meta-loss for the current step in a training loop as the loss in the next step, after performing an SGD step. Thus, the meta-loss, unlike the loss, decreases if a step on some data leads to improved performance on different data.
This meta-loss can be particularly useful for parts of the training algorithm that do not make sense to be directly optimized with the usual loss.
A particular example for this is the data distribution. While it was shown that it can be helpful to change the data distribution (over time) for SGD training, by training with a curriculum \citep[see e.g.][]{Bengio+chapter2007} or by filtering the dataset \citep[see e.g.][]{datasetpruningchen2019understanding}, we can generally not decide what data to train on using the same loss we use for model training:
this would simply encourage the model to focus on a particularly easy data distribution, since that would minimize the loss. The meta-loss on the other hand rewards generalization.
The example of learning a distribution over a dataset is also particularly interesting because it is not differentiable in general. Thus, we would not be able to compute meta-gradients, even if compute cost was no issue.


We evaluate our method for two settings: learning a distribution over data splits in a toy experiment 
and augmentation selection in a real-world application to CIFAR-10 and CIFAR-100 \citep{cifar}.

The over-arching contribution of this paper is to introduce a way of adapting meta parameters $\phi$ online during SGD training of the elementary network parameters $\theta$ that is efficient in terms of data, memory and time.
We present a method to train online for generalization and do so for non-differentiable distributions.
Specifically, our contributions are:
\begin{itemize}
    \item We present a mechanism to cheaply approximate meta-gradients in SGD, using the next batch as a proxy for the validation set.
    \item We introduce the Gradient-Alignment Reward setup with an efficient implementation, which allows us to use reinforcement learning with an unique per-example reward.
    \item We show promising empirical results for our approach on a simple example and a real-world application.
\end{itemize}

\begin{algorithm}
\begin{algorithmic}[1]
\STATE initialize parameters $\theta_1\color{red}, \phi_1, \phi_2:=\phi_1$
\FOR{$t \in \{1,\dots,T\}$}
\STATE $l_t \gets \stochasticloss(\theta_t{\color{red},\phi_t})$
\STATE $\theta_{t+1} \gets \theta_t - \alpha \nabla_{\theta_t}l_t$
{\color{red}\IF{$t > 1$}
\STATE  $\phi_{t+1} \gets \stopgrad{\phi_t - \beta \nabla_{\phi_{t-1}} l_t}$
\ENDIF}
\ENDFOR
\end{algorithmic}
\caption{In-Loop Meta-Learning}
\label{algorithm2}
\end{algorithm}

\section{Related Work}
A commonly-used approach in the literature (e.g., \citet{luketina-icml16a} or \citet{darts}) is to use \emph{alternating SGD}, which alternates SGD steps for $\theta$ w.r.t. training loss and for $\phi$ w.r.t. validation loss.
We follow a similar approach with two key differences: (i) we do not use a held-out dataset to train $\phi$ and (ii) we train weights $\phi$ parameterizing a distribution non-differentiable $p$.

Most closely related to our proposed method is the work by \citet{optdatausage_diffrewards} on data weighting.
In this work they propose a similar reward to ours here, with two key differences: (i) they only consider aligning with the gradients of validation examples and 
(ii) they only use an approximation to the example-wise alignment.
Counter, instead of using an approximation, we present a very efficient method of computing the example-wise alignments exactly (see Section~\ref{computing}).
Further, we discuss properties of the reward in detail (see Section~\ref{section:gradalignreward}), which prior has been omitted.

Other related work stems from the realm of online curriculum learning \citep[see e.g.][]{autocurriculumgraves}.
Similarly to our toy example, this line of work decides on what data to train as the training goes.
The main difference is that we directly optimize the data distribution using the aforementioned approximation of the meta-gradient.

\section{In-Loop Meta-Training}
\label{inloopmeta}


In contrast to previous work, in order to be more data-efficient and not require a validation set, we propose to exploit the fact that we are using SGD, and that we can use the next batch as a cheap proxy for the validation set.

Algorithm \ref{algorithm2} outlines the general approach. The standard SGD loop is shown in black. In each step the parameters $\theta$ are optimized to greedily minimize a stochastic loss $\stochasticloss(\theta)$. In red we extend this standard framework with meta-learning updates. The loss now depends additionally on the meta parameters $\phi$. We optimize $\phi$ not to minimize the loss directly, but to minimize the loss of the next step through the update performed in this step, as is done in standard unrolled gradient loop setups. 
In this setup we can have a direct dependence of the next-step loss on the meta parameters $\phi$ not through the update, since the module that $\phi$ parameterizes might be applied in that step, too (this is, e.g., the case if $\phi$ parameterizes the data augmentation strategy, a case we tackle in our experiments).
Since we propose to re-use the training steps as validation steps, we have a new dependency of $\phi$ compared to previous work, namely on the validation loss directly. This results, for differentiable setups, in the following meta gradient: 
\begin{align*}
\nabla_{\phi_{t-1}}l_{t} = \nabla_{\theta_t} l_t \cdot \nabla_{\phi_{t-1}} \theta_t + \nabla_{\phi_t} l_t \cdot \nabla_{\phi_{t-1}} \phi_t.
\end{align*}
The first term here describes the meta-gradient we are interested in: how to change $\phi_{t-1}$ such that the next update of $\theta$ improves the loss in the following step $t$.
The second term on the other hand describes how to change $\phi_{t-1}$ such that its update $\phi_{t}$ makes step $t$ as simple as possible; this does not facilitate generalization. We therefore propose to either just use a different loss in every second step, which does not depend on $\phi$ (e.g., in our case of $\phi$ parameterizing the data augmentation, just use a simple default augmentation in every second step), or cancel the second term artificially. In Algorithm \ref{algorithm2} we show the second option:
to cancel the term, we detach $\phi$ from the graph after each update, as is indicated in the algorithm by the stop gradient operation $sg$. The $sg$ function is defined as the identity, but with a zero gradient; so $sg(x)=x$ for all $x$, but $\frac{\partial sg(x)}{\partial x} = \mathbf{0}$.

While this setup is very general, it also is very expensive to compute the meta-gradient $\nabla_{\phi_{t-1}} l_t$.
As validation for this statement we performed a small experiment with a WideResNet-28-10 on an NVIDIA Tesla P100.
We average step times over one epoch of CIFAR10 training. We looked at two ways of computing the meta-gradients. (i) First, we used the \texttt{higher} library \citep{higher} for the meta-gradient computations. (ii) Second, we re-used the gradient $\nabla_{\theta_{t}} \stochasticloss_t$ to compute $\nabla_{\phi_{t-1}} \stochasticloss_t$ using the chain rule through an SGD update $\nabla_{\phi_{t-1}} \stochasticloss_t = -\alpha \nabla_{\theta_{t}} \stochasticloss_t \cdot \nabla_{\phi_{t-1}} \stochasticloss_{t-1}$. 
The optimized version improved the memory footprint slightly over the \texttt{higher} implementation, yielding a $2.7\times$ memory-increase compared to a $3.1\times$ memory-increase, but both had a very comparable step time-increase of around $6.4\times$.
Both the time and the memory overhead are a problem for large-scale machine learning training runs. A training run that finishes over the weekend without meta-gradients might take over half a month with this direct implementation of online meta-learning. Additionally, potential memory problems might require changes to the training pipeline as it would require nearly triple the memory.
Most importantly, back-propagation can only compute the derivative to fully differentiable parts of the training process. In this work we follow the setup of Algorithm \ref{algorithm2}, but propose a reward to approximate this gradient for any distribution.

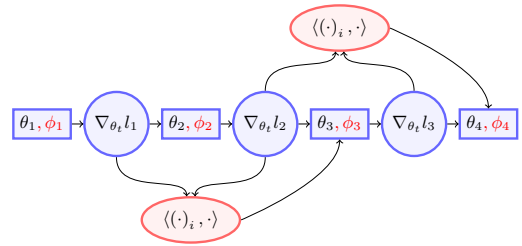
\begin{figure}[ht]
\centering
\scalebox{0.75}{
\begin{tikzpicture}[
font=\footnotesize,
redcircle/.style={ellipse, draw=red!60, fill=red!5, very thick, minimum size=5mm},
bluecircle/.style={circle, draw=blue!60, fill=blue!5, very thick, minimum size=5mm},
redsquare/.style={rectangle, draw=red!60, fill=red!5, very thick, minimum size=5mm},
bluesquare/.style={rectangle, draw=blue!60, fill=blue!5, very thick, minimum size=5mm}
]
\newcommand{\state}[1]{$\theta_#1\color{red}, \phi_#1$}
\node[bluesquare] (state1) {\state{1}};
\newcommand{\nextstate}[2]{\node[bluesquare] (state#2) [right=0.2cm of update#1] {\state{#2}};}

\newcommand{\update}[1]{\node[bluecircle] (update#1) [right=0.2cm of state#1] {$\nabla_{\theta_t}l_#1$};\draw[->] (state#1.east)  .. controls +(right:0mm) and +(left:0mm) .. (update#1.west);}
\update{1}
\nextstate{1}{2}
\update{2}
\nextstate{2}{3}
\update{3}
\nextstate{3}{4}

\draw[->] (update1.east) .. controls +(right:0mm) and +(left:0mm) ..  (state2.west);
\draw[->] (update2.east)  .. controls +(right:0mm) and +(left:0mm) .. (state3.west);
\draw[->] (update3.east)  .. controls +(right:0mm) and +(left:0mm) .. (state4.west);

\newcommand{\metaupdate}[3]{\node[redcircle] (metaupdate#2) [below= of state#2] {$\langle\left(\cdot\right)_i,\cdot\rangle$};
\draw[->] (update#1.south) .. controls +(down:7mm) and +(up:4mm) ..  (metaupdate#2.100);
\draw[->] (update#2.south) .. controls +(down:7mm) and +(up:4mm) ..  (metaupdate#2.80);
\draw[->] (metaupdate#2.east) .. controls +(right:0mm) and +(down:7mm) ..  (state#3.south);}
\newcommand{\metaupdateabove}[3]{\node[redcircle] (metaupdate#2) [above= of state#2] {$\langle\left(\cdot\right)_i,\cdot\rangle$};
\draw[->] (update#1.north) .. controls +(up:7mm) and +(down:4mm) ..  (metaupdate#2.260);
\draw[->] (update#2.north) .. controls +(up:7mm) and +(down:4mm) ..  (metaupdate#2.280);
\draw[->] (metaupdate#2.east) .. controls +(right:0mm) and +(up:7mm) ..  (state#3.north);}

\metaupdate{1}{2}{3}
\metaupdateabove{2}{3}{4}

\end{tikzpicture}
}
\caption{ A diagram outlining in-loop meta-learning with the \gar{}. Extensions to the standard SGD loop are red.}
\label{fig:dataflow}
\end{figure}

\section{Gradient-Alignment Reward}
\label{section:gradalignreward}

As discussed above there are cases where it is not possible to compute the meta-gradient directly, because we optimize some distribution $p$ that depends on the meta-parameters $\phi$ and produces hard samples $a_i \sim p(\cdot;\phi)$ for each data point. For example, $a_i$ could represent a network hyperparameter or the data sampling strategy. In these cases, a simple approximation to the meta-gradient is to use the REINFORCE trick \citep{williams1992reinforce} with the negation of the next step\textquotesingle{s} loss as reward $\overline{r_t} = -\loss(\theta_{t+1})$ using
\begin{equation*}
    \nabla_{\phi_t}\Exp_{a \sim p(\cdot;\phi_t)}[\loss(\theta_{t+1})] \approx  \overline{r_{t}}\cdot \sum_{i=1}^n \nabla_{\phi_t}\log p(a_i;\phi_t),
\end{equation*}
where $\loss(\theta_{t})$ is the batch loss of the $t$-th step and $a\sim~p(\cdot;\phi_t)$. We refer to this approximation of the meta-gradient as \emph{Next Step Loss Reward (NSLR)}. While this approximation is bias-free and simple, it incurs a lot of variance.
Further, it provides a single reward only, even though we sample from $p$ for each data point.
It would be far more effective to use a reward for each sample instead of one reward for the whole batch.
To achieve efficient online meta-learning, we propose the \emph{Gradient-Alignment Reward (\gar{})} which allows us to use reinforcement learning with a unique per-example reward.

\gar{} allows to do efficient in-loop meta-learning for large models and large datasets. It is computed mostly from artifacts of a standard SGD training and has little memory overhead. We define the \gar{} as the dot product of a current example-gradient with the next step\textquotesingle{s} gradient. Formally, the \gar{} $r_{t,i}$ for the $i$-th example in step {$t$} is\\
\begin{align}
	r_{t,i} = \left\langle \nabla_\theta \stochasticloss(\theta_t,\phi_t)_i, \nabla_\theta \loss(\theta_{t+1})  \right\rangle,
\end{align}\\
where $\stochasticloss(\cdot,\cdot)_i$ is the loss of the $i$-th example in a batch, $n$ is the batch size and $\loss(\theta_{t+1}) = \frac{1}{n}\sum_{j=1}^{n}\stochasticloss(\theta_{t+1},\phi_{t+1})_j$ is the batch loss of the $(t+1)$-th step.
We maximize this reward by sampling from a policy parameterized by $\phi$ for each example and compute a gradient estimate with the REINFORCE trick.
Using vanilla policy-gradient the gradient toward $\phi_t$ is estimated as
\begin{align}
    &\nabla_{\phi_t} \Exp_{a \sim p(\cdot;\phi_t)}[\loss(\theta_{t} - \alpha \nabla_{\theta_t} \frac{1}{n}\sum_{i=1}^{n} \stochasticloss(\theta_t,a_i)_i)] \\
    &\approx \sum_{i=1}^n r_{t,i} \cdot \nabla_{\phi_t}\log p(a_i;\phi_t),
\end{align}
where $a_i \sim p(\cdot;\phi_t)$.
In Figure~\ref{fig:dataflow} we visualize the computational flow of this update in comparison to the standard SGD training-loop.


The following theorem should give some intuition for the relationship of \gar{} with the unrolled gradient loop.

\begin{theorem}
The GAR update is an unbiased estimator of the meta-gradient (i.e. $\nabla_{\phi_{(t-1)}}\loss(\theta_t)$) in the infinite batch size limit.
\label{garetheorem}
\end{theorem}
\begin{proof}
Let our policy, trained with the meta-objective to maximize the GAR, be a distribution $\augpolicy$ that depends on $\phi$. Further, assume that the model, and therefore the loss, depends only on $\phi$ through samples from $\augpolicy$.
We can thus denote the example loss $\stochasticloss(\theta,\phi)_i$ as $\stochasticloss(\theta,a_i)_i$, where $a_i \sim \augpolicy(\phi)$.
We consider a standard SGD update $\theta_{t+1} := \theta_t - \alpha \nabla_\theta \frac{1}{n}\sum_{i=1}^{n}\stochasticloss(\theta_t,a_i)_i$ for some given model state $\theta_t$ and meta actions $a_i \sim \augpolicy(\phi)$. 
In the infinite batch size setting we can sample infinitely many meta actions $a$ per batch.
Thus, in the following we assume for the batch loss $l(a)$, which might depend on meta actions $a$, that
$l(a) = \Exp_{a' \sim \augpolicy(\phi_t)}[l(a')]$, for $a_i \sim \augpolicy(\phi)$. This is trivially fulfilled for infinite batches of the form $\Exp_{a' \sim \augpolicy(\phi_t), \stochasticloss'}[\stochasticloss'(\theta_t,a')]$.
We denote the loss for the meta-gradients as $\loss(\theta_{t+1})$.
From this we can infer that the distributional gradient of the update of an algorithm trained with the \gar{} and the REINFORCE trick point in the same direction:
\begin{align*}
&\nabla_{\phi_t} \loss(\theta_{t+1})\\
\intertext{Using the definition of $\theta_{t+1}$ and the infinite batch assumption.}
&= \nabla_{\phi_t} \loss(\theta_t - \alpha \nabla_{\theta_t} \frac{1}{n}\sum_{i=1}^{n} \Exp_{a_i \sim p(\cdot;\phi_t)}[\stochasticloss(\theta_t,a_i)_i])\\
\intertext{Apply the chain rule.}
&= -\alpha\nabla_{\theta_{t+1}} \loss(\theta_{t+1}) \cdot \nabla^2_{\phi_t,\theta_t} \frac{1}{n}\sum_{i=1}^{n} \Exp_{a_i \sim p(\cdot;\phi_t)}[\stochasticloss(\theta_t,a_i)_i]\\
\intertext{Now we re-arrange sums, expectations and gradients}
&= -\alpha\nabla_{\theta_{t+1}} \loss(\theta_{t+1}) \cdot \frac{1}{n}\sum_{i=1}^{n}\nabla_{\phi_t} \Exp_{a_i \sim p(\cdot;\phi_t)}[\nabla_{\theta_t} \stochasticloss(\theta_t,a_i)_i]\\
\intertext{Make use of the REINFORCE trick.}
&= -\alpha\nabla_{\theta_{t+1}} \loss(\theta_{t+1}) \cdot \nonumber\\
&\qquad \frac{1}{n}\sum_{i=1}^{n}\Exp_{a_i \sim p(\cdot;\phi_t)}[\nabla_{\theta_t} \stochasticloss(\theta_t,a_i)_i \cdot  \nabla_{\phi_t}\log p(a_i;\phi_t)]\\
\intertext{In the limit of the infinite batch assumption $\theta_{t+1}$ does not depend on $a$, since we take expectations over $a$ and do not only sample.}
&= -\alpha  \Exp_{a \sim p(\cdot;\phi_t)}\Bigg[\frac{1}{n}\sum_{i=1}^{n}\langle \nabla_{\theta_t} \stochasticloss(\theta_t,a_i)_i,\nabla_{\theta_{t+1}}\loss(\theta_{t+1},\phi_{t+1})\rangle \cdot \nonumber\\
&\qquad \qquad \qquad \qquad \nabla_{\phi_t}\log p(z;\phi_t)\Bigg]\\
\intertext{Finally we apply the definition of the \gar{} $r$.}
&= -\frac{\alpha}{n}\Exp_{a \sim p(\cdot;\phi_t)}\left[\sum_{i=1}^{n}r_{t,i} \cdot \nabla_{\phi_t}\log p(z;\phi_t)\right].
\end{align*}
\end{proof}

The above proof shows that comparing example gradients from the current step with the aggregated gradient of the next step is a bias-free estimator of the stochastic meta-gradient in the infinite batch size limit.

The \gar{} will only consider the impact of $\phi_{t}$ on the update generated with the last batch and will not consider the impact of $\phi_t$, like noted in Algorithm \ref{algorithm2} by the stop gradient $sg$.

Our method just requires computing the dot products of gradients, besides computing the gradients inside $\augpolicy$ on top of the terms which are anyways needed for an SGD loop.
%
%
%
 Similar setups where proposed before, but in the following section we also detail how to compute the \gar{} efficiently. With our optimized implementation, we empirically incur an increase in training time of less than 25\% (a stark improvement over the direct gradient computation which had an overhead of around 540\%) 
and a memory overhead of less than 80\% (a $2\times$ improvement) in the same setting as used for comparison in Section \ref{inloopmeta}.

\section{Efficiently Computing the \gar{}}
\label{computing}
This section details the efficient computation of the \gar{}, the method would work without the following strategies, but not as fast and memory-saving.

The \gar{} is the dot product between a batch gradient and an example gradient. In this section we assume we are given some arbitrary batch gradient $g$ and compute the alignment of it with each example gradient of a given batch.
To compute the gradient-alignment reward efficiently we use the BACKPACK package \citep{backpack}, which gives us easy access to the incoming gradients of each layer.
The full gradients of a model are a concatenation of the weights of multiple layers. The dot product between two full model gradients is thus the sum of the dot products between the weights of their respective layers. Below we show how we compute gradient dot products for the three main weight types in neural networks. We refer to the batch size as $n$ and arbitrary dimensions that depend on the model as $d_i$ for some integer $i$.

In our derivations we use the per-example gradient $\partial \stochasticloss_i / \partial w$ of a weight $w$, which is only cleanly defined as part of the batch gradient if there is no interaction between the example computations in a batch. This is the case in most current neural networks if batch normalization is not used.

\paragraph{Biases} Biases in neural networks are a simple vector addition of a bias vector $b \in \reals^{d_1 \times \dots \times d_k}$ to each hidden state $x_i$ in a batched hidden state $x \in \reals^{n \times d_1 \times \dots \times d_k}$. The computation performed with a bias is $x_i' = x_i + b$. We receive the incoming gradient $\partial \stochasticloss / \partial x' \in \reals^{n \times d_1 \times \dots \times d_k}$ from PyTorch\textquotesingle{s} autograd \citep{pytorch}. We simply compute $\langle (\partial \stochasticloss / \partial x')_i, g \rangle$ for all $i \in \{1,\dots,n\}$ sequentially. Biases are usually only small, thus this is not very expensive in general.

\paragraph{Linears} Linears perform matrix multiplications between a batch of incoming hidden vectors $x \in \reals^{n \times d_1}$ and a weight matrix $w \in \reals^{d1 \times d2}$ to yield a new batch of hidden vectors $x' = x \cdot w \in \reals^{n \times d_2}$. The per-example gradient for the $i$-th example can be computed as the outer-product $\partial \stochasticloss_i / \partial w = x_i \cdot (\partial \stochasticloss / \partial x')_i^\T$, where $\partial \stochasticloss / \partial x' \in \reals^{n \times d_2}$ is computed by autograd. 
\begin{lemma}
The dot product between the per-example gradient $\partial \stochasticloss_i / \partial w$ and the given next batch gradient $g$ is $ (x_i^\T \cdot g) \cdot (\partial \stochasticloss / \partial x')_i$.
\end{lemma}
\noindent{}The proofs of Lemmas 2 and 3 are given in Appendix \ref{app:proofs}.

Based on Lemma 2, for a given batch we can compute the dot product with the element-wise gradient as $(x \cdot g) \cdot (\partial l / \partial x')^\T$. This way of computing the dot products is only marginally more expensive than a forward pass and requires much less memory than computing the products sequentially.
A similar derivation for computing gradient norms was previously shown by \citet{backpack}.

\paragraph{Convolutions} 
Convolutions are an essential component of most neural network architectures for vision tasks. In the following we will only discuss the single channel case with implicit zero-padding and a stride of one for simplicity, but the analysis extends to many channels, other padding strategies and strides. To introduce our highly optimized implementation, we first remember the operation a convolution with a convolution matrix $K \in \reals^{c_1 \times c_2}$ performs on an input $x \in \reals^{n_1 \times n_2}$. Unlike previously, for this optimization it is enough to consider a single example of size $n_1 \times n_2$ instead of a batch. To simplify the problem, we assume the height and width of the kernel, $c_1$ and $c_2$, to be odd. Since we use zero-padding we assume in the following calculations that out of bounds indexes yield zeros, or equivalently that $x$ has an all zero frame of widths $\lfloor\frac{c_1}{2}\rfloor$ and $\lfloor\frac{c_2}{2}\rfloor$ in dimension 2 and 3. We can now define the convolution operator as 
\begin{align*}
    &x'_{i_1,i_2} = (x * K)_{i_1,i_2} \\
    &= \sum_{j_1=1}^{c_1} \sum_{j_2=1}^{c_2} K_{j_1,j_2} \cdot x_{i_1+j_1-\lfloor\frac{c_1}{2}\rfloor,i_2+j_2-\lfloor\frac{c_2}{2}\rfloor},
\end{align*}
where we denote indexes as function arguments. Further we recall that the gradient of a convolution towards its weight $\frac{\partial \stochasticloss}{\partial K}$ can be defined as
\begin{align*}
    \frac{\partial \stochasticloss}{\partial K_{j_1,j_2}} = \sum_{i_1=1}^{n_1} \sum_{i_2=1}^{n_2} \frac{\partial \stochasticloss}{\partial x'_{i_1,i_2}} \cdot x_{i_1+j_1-\lfloor\frac{c_1}{2}\rfloor,i_2+j_2-\lfloor\frac{c_2}{2}\rfloor}.
\end{align*}
\begin{lemma}
Given the above assumptions on the convolution function, we have that the dot product of a given matrix $g \in \reals^{c_1 \times c_2}$ with each per-example gradient can be written as $\langle g, \frac{\partial \stochasticloss_i}{\partial K}\rangle = \sum_{i_1=1}^{n_1} \sum_{i_2=1}^{n_2} \frac{\partial \stochasticloss}{\partial x'_{i_1,i_2}}\cdot 
    (x * g)_{i_1,i_2}$

\end{lemma}

The new form of the dot product simply consists of an element-wise product between the incoming gradient and the result of using $g$ instead of $K$ in the convolution. Thus, we can compute the dot product between $g$ and the gradient of each individual example in a batch very cheaply. The costs are the same as the forward pass through the convolution except for the final dot product.

\section{Experiments}
\label{experiments}
We performed experiments in two different setups. We will first detail an interpretable and easy-to-reproduce toy experiment, than we will detail the application of the \gar{} of in-loop meta learning of augmentations.

\begin{table}[b]
    \centering
    \scalebox{0.85}{
    \begin{tabular}{llcc}
    \toprule
    Net                        & Method & Noisy Split  & Other Splits\\
    \midrule
    \multirow{2}{*}{FC$_{BN}$} & GAR     & $0.61$ & $0.93$\\
                            {} & NSLR & $0.96$ & $0.89$\\
    \midrule
    \multirow{2}{*}{FC} & GAR     & $0.65$ & $0.93$\\
                     {} & NSLR & $0.86$ & $0.90$\\
    \midrule
    \multirow{2}{*}{CNN$_{BN}$} & GAR     & $0.69$ & $0.92$\\
                            {} & NSLR  & $0.92$ & $0.90$\\
    \bottomrule
    \end{tabular}
    }
    \caption{Average AUC for the split probability over $10$ epochs of (non-)noisy splits for different architectures.}
    \label{tab:toy_averaged}
\end{table}

\begin{figure*}[t]
    \begin{subfigure}{.32\linewidth}
    \includegraphics[width=\linewidth]{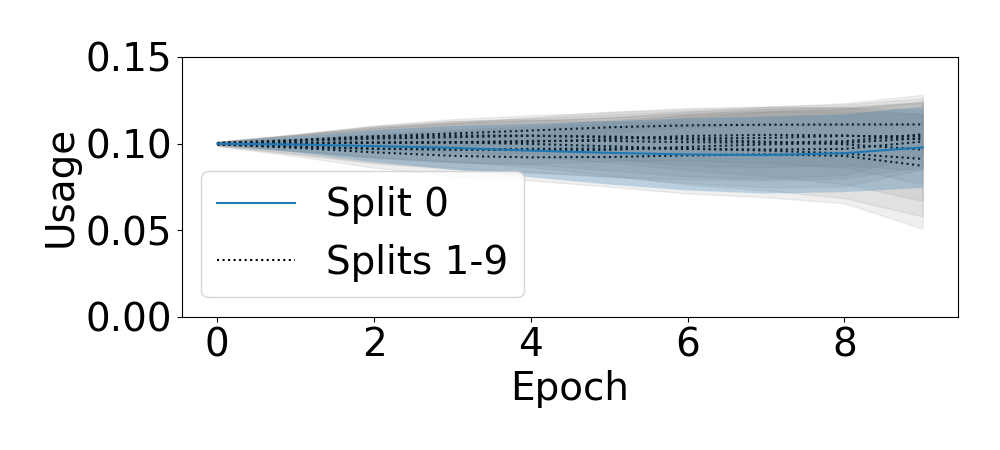}
    \subcaption{}
    \label{nslrbaseline}
    \label{baseline}
    \end{subfigure}
    \begin{subfigure}{.32\linewidth}
    \includegraphics[width=\linewidth]{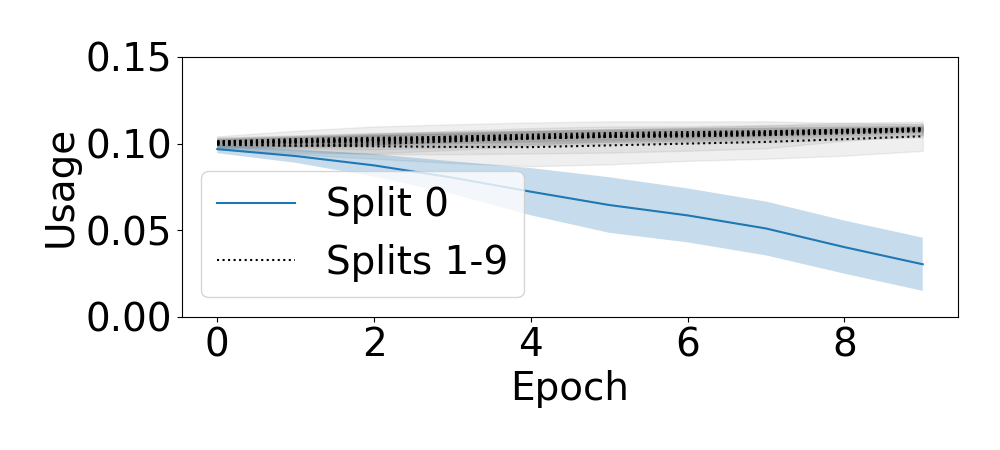}
    \subcaption{}
    \label{gare}
    \end{subfigure}
    \begin{subfigure}{.32\linewidth}
    \includegraphics[width=\linewidth]{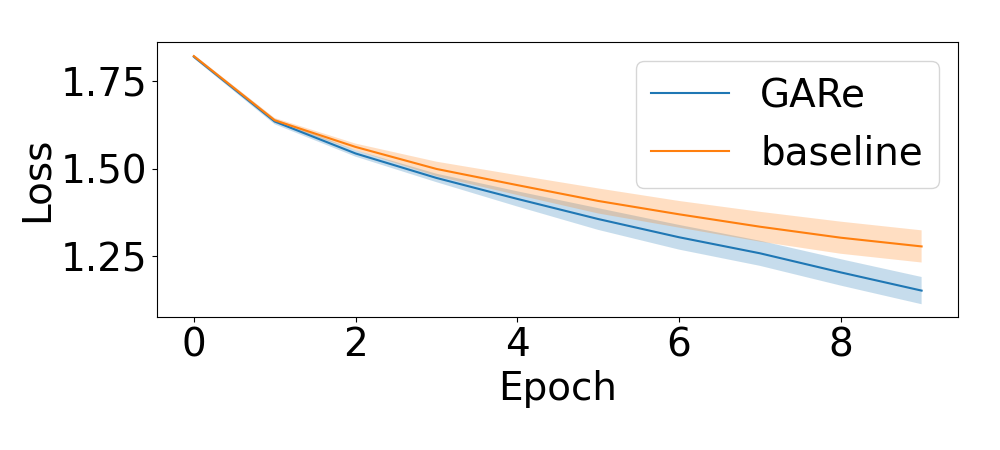}
    \subcaption{}
    \label{losses}
    \end{subfigure}
    \caption{Average split distribution of the NSLR baseline (\protect\subref{baseline}) and \gar{} (\subref{gare}) across epochs. Training losses are given in (\subref{losses}).}

    \label{fig:toyexpfigue}
\end{figure*}

\subsection{Illustrative Example: Batch Sampling Distribution}
We perform a motivational experiment on a small toy task. We split the CIFAR-10 dataset into 10 equally sized parts and on the zeroth part we replace all labels with labels drawn uniformly at random. Therefore, on split 0 most labels are wrong. Things learned from split 0 will, thus, generalize badly to other splits. For each example in an SGD batch, we first sample the split from a learned distribution and then uniformly sample from within that split.
We make experiments with three networks: a fully-connected network with a single hidden layer of size 200 (FC), the same fully-connected network but with BatchNorm \citep{ioffe-icml15} (FC$_{BN}$) and a small CNN (CNN$_{BN}$) with a single convolution of size 3 from the input to a single channel, followed by BatchNorm and a linear layer. In all networks we use ReLU activations between layers.
We use a batch size of 1000, train for 10 epochs and apply an SGD optimizer with Nesterov momentum of 0.9, a .0005 L$_2$ regularization and a fixed learning rate of 0.1.
As for the meta optimization, we parameterize the distribution over splits by an "inverted" softmax distribution generated from learned logits $s$, that is $p(s)_i \sim 1 - \softmax(s)_i$. This is useful compared to other distributions, since we saw that, while softmax has a bias towards a single winner, this has a tendency towards a single loser.
To train $s$ online with \gar{} we use policy gradient. We aggregate gradients over 10 steps, normalize the rewards for each step and use Adam \cite{adam} with a learning rate of 0.1.
To compare our method fairly. We compare it to NSLR, as proposed in section \ref{section:gradalignreward}, which is much simpler, but still novel. While there is only one NSLR reward per step, this baseline yields a bias free estimator of the true stochastic meta-gradient. We performed all experiments with 10 different seeds.
In Figure \ref{fig:toyexpfigue} we show the allocation of $p(s)$ over time for \gar{} and the NSLR, as well as training losses over time, the usage of each bucket, for the fully-connected network with batch norm. We can see that \gar{} presses down the noisy category with little variance in Figure \ref{gare}, while the baseline does not succeed in omitting the noisy data, as we see in Figure \ref{nslrbaseline}. In Figure \ref{losses} we see how this impacts training losses of the two setups.
We could see similar results for the other setups and show the area under the curve (AUC) of the split probability (the usage) for these in Table \ref{tab:toy_averaged}. An optimal method would show a low AUC for the noisy split, but a high AUC for the other splits. It would, thus omit training on noisy data, but not omit training on any of the non-noisy splits.
To experiment with different networks, setups or add noise to the images instead of the labels, we refer to our public colab notebook\footnote{\url{https://bit.ly/34K7aAT}} generating these experiments with no setup.

\begin{table}[h]
\setlength\tabcolsep{2pt}
\centering
\footnotesize
\begin{tabular}{l|cccc|cc}
  & PBA & Fast AA & AA  &  RA & OLA (WL) & OLA (RA)\\
  \hline 
  \textbf{CIFAR-10} &&&&&\\
   UA Baseline & - & - & - & - &97.61 $\pm$ 0.16 & 97.51 $\pm$ 0.18\\
   Method & 97.4 & 97.3 & 97.4 & 97.3 & 97.39 $\pm$ 0.15 & 97.56 $\pm$ 0.07 \\ 
  \hline
  \textbf{CIFAR-100} &&&&&&\\
    UA Baseline & - & - & - & - &83.20 $\pm$ 0.34&83.40 $\pm$ 0.09\\
   Method & 83.3 & 82.7 & 82.9 & 83.3 & 84.30 $\pm$ 0.40 & 83.54 $\pm$ 0.14 \\ 
  \hline  
\end{tabular}
\vspace{0.2cm}
\caption{Results of the OLA experiments. The results are test accuracies and the 95\% confidence interval is noted with $\pm$.}
\label{tab:aug_results}  
\end{table}

\subsection{Online-Learned Augmentation-Strategy}
The choice of image augmentation was shown to have enough impact to improve performance considerably \cite{autoaugment}. We apply the \gar{} to learn the augmentation policy online. We refer to this method as \emph{Online-Learned Augmentation-Strategy (OLA)}.
We use a setup inspired by RandAugment \citep{randaugment}. We have a set of augmentations $\mathcal{A}$ and for each image we sample a learned number $r \in \{1,\dots,4\}$ of augmentations $a_1,\dots,a_r \in \mathcal{A}$ uniformly without replacement. $r$ itself is sampled from a learned distribution $p(r) = \softmax(l^{(r)})$.
We apply each augmentation $a_i$ in sampling order to the example image, each with a sampled strength $k_{a_i} \in \{0,\dots,30\}$. $k_{a_i}$ is sampled from the distribution $p(k_{a_i}) = \softmax(l^{(a)}_{a_i})$, which depends on the applied augmentation $a_i$.
Each sampled augmentation $a_i$ is actually applied with a learned probability $p(d_{a_i}) = \sigma(l^{(d)}_{a_i})$. A sampled augmentation might thus not be applied after all.
In the above all logits $l^{(r)} \in \reals^{4}$, $l^{(a)} \in \reals^{|\mathcal{A}| \times 31}$ and $l^{(d)} \in \reals^{|\mathcal{A}|}$ are learned weights and initialized to zero.
We outline the augmentation sampling process in Algorithm~\ref{augmentationsampling}.

\begin{algorithm}[tb]
\begin{algorithmic}[1]
\STATE receive sample image $x$
\STATE sample the number of augmentations $r \sim p(r)$
\STATE sample the augmentations to apply $a_1,\dots,a_r$ uniformly at random without replacement from $\mathcal{A}$
\FOR{$i \in \{1,\dots,r\}$}
\STATE sample keep indicator $d_{a_i} \sim p(d_{a_i})$
\IF{$d_{a_i}$}
\STATE sample strength $k_{a_i} \sim p(k_{a_i})$
\STATE apply the augmentation $x \gets a_i(x,k)$
\ENDIF
\ENDFOR
\end{algorithmic}
\caption{OLA Augmentation Sampling Procedure}
\label{augmentationsampling}
\end{algorithm}

We perform our evaluations on CIFAR-10 and CIFAR-100 using a WideResnet-28-10 \citep{wrn} and follow the setup of \citet{randaugment} in detail, but we use a larger batch size of 256. We keep the number of epochs stable, though, such that we do not have an unfair advantage. To prevent the augmentation distribution from collapsing we interleave steps with learned augmentations with steps that do not use augmentations. Such that we consider alignments of augmented examples with a non-augmented batch. For maximal efficiency we use the gradients computed on non-augmented batches not only to compute the \gar{}, but treat them as normal steps in the main training and update the neural network with them.
Like before we use Adam for the meta optimization. We set the learning rate of Adam to $0.1$, normalize the rewards in each batch and do not aggregate meta gradients.

\begin{algorithm}[tb]
\begin{algorithmic}[1]
\STATE receive sample image $x$
\STATE sample the augmentations to apply $a_1$ and $a_2$ uniformly at random without replacement from $\mathcal{A}$
\FOR{$i \in \{1,2\}$}
\STATE sample keep indicator $d \sim \text{Bern}(0.5)$
\IF{$d$}
\STATE sample strength $k$ uniformly at random from $\{0,\dots,30\}$
\STATE apply the augmentation $x \gets a_i(x,k)$
\ENDIF
\ENDFOR

\end{algorithmic}
\caption{UA Augmentation Sampling Procedure}
\label{uareimp}

\end{algorithm}

We compare OLA with a set of common augmentation strategies, most of which have to pre-trained. We compare with Population based Augmentation \citep[PBA;][]{ho2019population}, AutoAugment~\citep[AA;][]{autoaugment}, Fast AA \citep{fastaa} and RandAugment~\citep[RA;][]{randaugment}.
So far in the literature little attention was given to the search space. Most previous work use slightly different search spaces. PBA, AA and RandAugment all have slightly different search spaces for example. In our experiments we found the search space choice to be important. 
We reimplemented UniformAugment~ \citep[UA;][]{uniformaugment} and found that surprisingly in our setup we could considerably improve the performance of our reimplementation of UA depending on the search space, something we were not able to do for RA. Therefore, unlike previous work, we provide our reimplementation of UA as an additional baseline, since it was evaluated under the exact same settings. Algorithm \ref{uareimp} outlines our reimplementation of UA. We made two main changes: We sample strengths $k$ from a range of integer values instead of a real-valued range, to align with our comparisons, and we sample the augmentations without replacement, which makes the set of applied augmentations more diverse. Other than that, we removed the double sampling of the keep probability and replaced with an equivalent single sampling.

In Table \ref{tab:aug_results} we show the average test accuracies over 5 runs of our method and the UA baselines with confidence bounds and the comparisons from literature. We evaluated our method on two different search spaces, which we denote in parentheses for our evaluations. The search space RA is equivalent to the search space used for RandAugment, while WideLong (WL) is a search space that includes more extreme strength settings and more augmentations, see Table \ref{tab:WLsearchspace}. While our method and its baseline perform well in comparison to previous methods, the comparison we want to focus on is the comparison with our reimplementation of UA, since we share all setup with it. For CIFAR-10 one can see here, that we are performing similar to the strong baselines for both search spaces. For CIFAR-100 we outperform the baselines in both cases by at least the 95\% confidence interval. Note, that unlike previous methods we do not pre-train or do hyper-parameter search on a per-dataset basis. 
Our results point out that over different setups our method works either comparable or better than previous methods and the UA baseline, while other methods like AA, are expensively pre-trained for each dataset. Also worthy of mention is how well the UA baseline performs, even for the larger search space, compared to learned methods.

\section{Conclusion}
We presented a way of adapting meta parameters $\phi$ online during SGD training of the elementary network parameters $\theta$ that is efficient in terms of data, memory and time and applies to optimizing non-differentiable distributions during training. Key to our approach is the Gradient Alignment Reward, which allows using reinforcement learning with unique per-sample rewards.
We showed its benefits on an interpretable toy task and a real world task. This method has many potential future applications like large scale learned curricula or neural architecture search.



\paragraph*{\textsc{Acknowledgements}} The authors acknowledge funding by the Robert Bosch GmbH.%

\bibliography{bib/shortstrings,iclr2020/iclr2020_conference.bib,bib/lib,bib/shortproc}

\appendix
\section{Proofs of Lemmas 2 and 3}\label{app:proofs}
\begin{proof}[Proof of Lemma 2]
\begin{align*}
    \langle \partial \stochasticloss_i / \partial w, g \rangle &= \sum_{k_1 = 1}^{d_1}\sum_{k_2 = 1}^{d_2} (\partial \stochasticloss_i / \partial w)_{k_1,k_2} \cdot g_{k_1,k_2}\\
    &= \Tr ((\partial \stochasticloss_i / \partial w)^\T \cdot g) \\
    &= \Tr ((x_i \cdot (\partial \stochasticloss / \partial x')_i^\T)^\T \cdot g) \\
    &= \Tr (((\partial \stochasticloss / \partial x')_i \cdot x_i^\T) \cdot g) \\
    &= \Tr ((\partial \stochasticloss / \partial x')_i \cdot (x_i^\T \cdot g)) \\
    &= \langle  (\partial \stochasticloss / \partial x')_i, (x_i^\T \cdot g)\rangle.
\end{align*}
\end{proof}

\begin{proof}[Proof of Lemma 3]
\begin{align*}
    &\sum_{j_1=1}^{c_1} \sum_{j_2=1}^{c_2} \frac{\partial \stochasticloss}{\partial K_{j_1,j_2}} \cdot g_{j_1,j_2} \\
    &= \sum_{j_1=1}^{c_1} \sum_{j_2=1}^{c_2} \sum_{i_1=1}^{n_1} \sum_{i_2=1}^{n_2} \frac{\partial \stochasticloss}{\partial x'_{i_1,i_2}} \cdot x_{i_1+j_1-\lfloor\frac{c_1}{2}\rfloor,i_2+j_2-\frac{c_2}{2}} \cdot g_{j_1,j_2}\\
    &= \sum_{i_1=1}^{n_1} \sum_{i_2=1}^{n_2} \frac{\partial \stochasticloss}{\partial x'_{i_1,i_2}} \sum_{j_1=1}^{c_1} \sum_{j_2=1}^{c_2}  x_{i_1+j_1-\lfloor\frac{c_1}{2}\rfloor,i_2+j_2-\lfloor\frac{c_2}{2}\rfloor} \cdot g_{j_1,j_2}\\
    &= \sum_{i_1=1}^{n_1} \sum_{i_2=1}^{n_2} \frac{\partial \stochasticloss}{\partial x'_{i_1,i_2}}
    \cdot (x * g)_{i_1,i_2}.
\end{align*}
\end{proof}

\section{WideLong Search Space}
\begin{table}[h!]
\centering
\scalebox{.85}{
    \begin{tabular}{c|c  c|c}
            PIL operation & range &  PIL operation & range \\
\cmidrule(lr){1-2} \cmidrule(lr){3-4}
            identity& - & auto\_contrast& - \\
            equalize& 0.01 - 2.0 & rotate& $-135^{\circ}$ - $135^{\circ}$ \\
            solarize& 0 - 256 & color& 0.01 - 2.0\\
            posterize& 0.01 - 2.0 & contrast& 0.01 - 2. \\
            brightness& 2 - 8 & sharpness& 0.01 - 2.0\\
            shear\_x& 0.0 - 0.99 & shear\_y& 0.0 - 0.99\\
            translate\_x& 0 - 32 & translate\_y& 0 - 32\\
            blur& - & invert& -\\
            flip\_lr& - & flip\_ud& - \\
            cutout& 0 - 19 & & \\
    \end{tabular}
    }
    \caption{The WideLong (WL) search space. All methods are defined as part of Pillow (\url{https://github.com/python-pillow/Pillow}), as part of ImageEnhance, ImageOps or as image attribute, besides cutout \citep{devries2017improved}.}
    \label{tab:WLsearchspace}
\end{table}


\end{document}